\pdfoutput=1

\documentclass[11pt]{article}

\usepackage[]{acl}
\usepackage[hang,flushmargin]{footmisc}
\usepackage{balance}
\usepackage{changepage}

\usepackage{times}
\usepackage{latexsym}

\usepackage[T1]{fontenc}

\usepackage[utf8]{inputenc}

\usepackage{microtype}

\usepackage{inconsolata}

\usepackage{amsmath}
\usepackage{amssymb}
\usepackage{amsthm}
\usepackage{bm}
\usepackage{booktabs}
\usepackage{graphicx}
\usepackage{multicol}
\usepackage{multirow}
\usepackage{blindtext}
\usepackage{dirtree}
\usepackage{tikz}

\newcommand*\circled[1]{(#1)}
\usepackage{caption}
\usepackage{subcaption}

\newcommand{\parheader}[1]{{\bf \smallskip \noindent #1.}}

\newcommand{\DONE}[1]{\noindent \textcolor{green}{\textbf{DONE}}\\ }

\newcommand{\bs}{\boldsymbol}
\newcommand{\pee}{\mathbb{P}}
\DeclareMathOperator*{\argmin}{argmin}

\DeclareMathOperator*{\inv}{inv}

\newtheorem{definition}{Definition}[section]
\newtheorem{proposition}{Proposition}[section]
\newenvironment{proofsketch}{%
  \proof}{\endproof}

\newenvironment{llmprompt}{%
\vspace{0.5mm}\it%
\begin{adjustwidth}{3mm}{0mm}
}{\vspace{0.5mm}
\end{adjustwidth}
}

\title{Found in the Middle:\ Permutation Self-Consistency Improves\\ Listwise Ranking in Large Language Models}

\author{Raphael Tang,$^{*1}$ Xinyu Zhang,\thanks{~~Equal contribution.}$^2$ Xueguang Ma,$^2$ Jimmy Lin,$^2$ Ferhan Ture$^1$ \vspace{1mm}\\
$^1$Comcast AI Technologies~~~$^2$University of Waterloo\\
{\small $^1$\texttt{{\{raphael\_tang, ferhan\_ture\}}@comcast.com}~~~$^2$\texttt{\{x978zhan, x93ma, jimmylin\}@uwaterloo.ca}}}

\begin{document}
\maketitle
\begin{abstract}
Large language models (LLMs) exhibit positional bias in how they use context, which especially affects listwise ranking.
To address this, we propose \textit{permutation self-consistency}, a form of self-consistency over the ranking list outputs of black-box LLMs.
Our key idea is to marginalize out different list orders in the prompt to produce an order-independent ranking with less positional bias.
First, given some input prompt, we repeatedly shuffle the list in the prompt and pass it through the LLM while holding the instructions the same.
Next, we aggregate the resulting sample of rankings by computing the central ranking closest in distance to all of them, marginalizing out prompt order biases in the process.
Theoretically, we prove the robustness of our method, showing convergence to the true ranking under random perturbations.
Empirically, on five datasets in sorting and passage reranking, our approach improves scores from conventional inference by up to 34--52\% for Mistral, 7--18\% for GPT-3.5, 8--16\% for LLaMA v2 (70B).
Our code is at \url{https://github.com/castorini/perm-sc}.
\end{abstract}

\section{Introduction}
Large language models (LLMs) respond cogently to free-form textual prompts and represent the state of the art across many tasks~\cite{zhao2023survey}.
Their quality, however, varies with nuisance positional factors such as prompt order and input length.
As a descriptive example, consider this prompt:
\begin{llmprompt}
Arrange the following passages in decreasing relevance to the query, ``what are shrews?''\\[0.2ex]
\noindent (1) Cats hunt small mammals, such as shrews ...\\
\noindent (2) Shrews are mole-like mammals, widely ...\\
\noindent (3) Shrews use their noses to find prey and ...
\end{llmprompt}
The correct output order is $(2, 3, 1)$, from most to least relevant, but several positional biases may interfere with the model.
\citet{liu2023lost} demonstrate that LLMs tend to get ``lost in the middle'' of a long context and use the middle portion poorly, which suggests that the middle passage ($2$) in the example may get misranked (e.g., $3, 1, \underline{2}$).
\citet{wang2023large} find prompt order to affect quality, with some orders outperforming others;\ if items 1 and 3 were swapped in the prompt, the LLM would perhaps generate the mistaken ranking $(2, \underline{1}, \underline{3})$.

\begin{figure}
    \centering
    \includegraphics[scale=0.8,trim={0mm 1mm 0mm 1mm},clip]{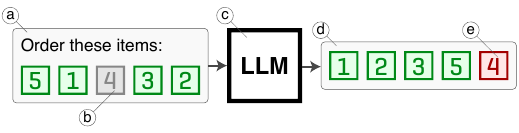}
    \caption{The conventional decoding process for listwise ranking with input prompt \circled{a}, language model \circled{c}, and output ranking \circled{d}. The grey item \circled{b} is ``lost in the middle'' by the LLM, resulting in its misranking \circled{e}.}\label{fig:ex1}
\end{figure}
\begin{figure}
    \centering
    \includegraphics[scale=0.56,trim={0mm 1mm 0mm 1mm},clip]{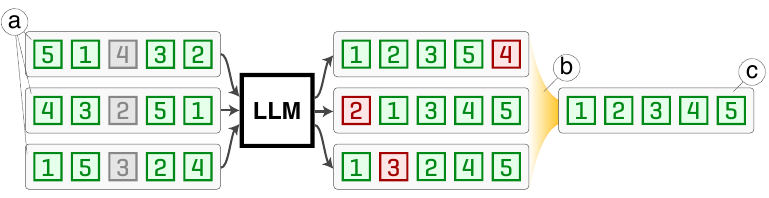}
    \caption{Our permutation self-consistency process. With the instruction fixed, we shuffle the input list for prompts \circled{a}, producing outputs with different mistakes. We aggregate \circled{b} these output rankings into one \circled{c}.}\label{fig:ex2}
\end{figure}

In this paper, we mitigate positional biases for listwise-ranking LLMs.
We propose \textit{permutation self-consistency}, a novel decoding strategy for improving the quality, consistency, and prompt-order invariance of black-box LLMs.
First, we construct prompts with randomly permuted input lists, then feed them into an LLM to generate a set of output rankings.
Then, we aggregate these outputs into the central ranking that minimizes the Kendall tau distance to all of them, marginalizing out prompt order as a factor;\ see Figures~\ref{fig:ex1} and \ref{fig:ex2}.
As related work, \citet{stoehr2023unsupervised} train direction-unaware probes on the representations of language models to detect order consistency, but their evaluation reveals the ranking direction of test examples to the model, deviating from standard practices.

Next, we assess the effectiveness of \mbox{permutation} self-consistency, both theoretically and empirically.
Theoretically, we prove in Section~\ref{sec:our-approach:theoretical-guarantees} that it recovers the true ranking under arbitrary noise distributions with enough observations and at least one correctly ordered pair in each observation.
Experimentally, we apply our method to tasks in math and word sorting, sentence ordering, and passage reranking~(\citealp{craswell2020overview}, \citeyear{craswell2021overview}), consistently increasing the scores of GPT-3.5, GPT-4, and LLaMA v2 (70B;\ \citealp{touvron2023llama}) by up to 4--17\%, 9--24\%, and 8--16\%, respectively.
We achieve similar gains for Mistral~\cite{jiang2023mistral} and Zephyr~\cite{tunstall2023zephyr}.
We conclude that permutation self-consistency improves listwise ranking in LLMs.
In line with our premises, we observe positional bias, as shown in Section~\ref{sec:position-bias}.

Finally, we conduct auxiliary analyses to justify our design choices.
In Section~\ref{sec:hyperparameter}, our hyperparameter study finds that quality quickly rises with the number of aggregated output rankings:\ the score improvement from using five aggregated rankings reaches 67\% of twenty, on average, suggesting that a few suffice for quality gain.
We further demonstrate that sampling temperature is ineffective for us, unlike the original self-consistency work~\cite{wang2023selfconsistency} in chain-of-thought reasoning, likely because listwise ranking does not require exploration of various reasoning paths.

Our contributions are as follows:\ \textbf{(1)} we propose a novel decoding technique for improving the quality, consistency, and position invariance of black-box, listwise-ranking LLMs;\ \textbf{(2)} we empirically establish the validity of our method in sorting and passage reranking on seven models and five datasets, and we theoretically prove the robustness of our method to certain classes of ranking noise, including ``lost-in-the-middle'' type ones;\ and \textbf{(3)}~we provide new analyses on positional biases in listwise-ranking LLMs, finding that biases depend on pairwise positions of items in the list.

\section{Our Approach}

\subsection{Preliminaries}
\parheader{Notation}
We define an $n$-\textit{ranking} as a permutation $\sigma: \{1, \dots, n\} \mapsto \{1, \dots, n\}$.
For some sequence $\bs X := \{X_i\}_{i=1}^n$, define $\bs X[\sigma]$ as the permuted sequence of $\bs X$ transformed by $\sigma$, where $\bs X[\sigma]_i := X_{\sigma(i)}$.
Let the inversion vector of $\sigma$ be
\begin{equation}
    \inv(\sigma)_i := \#\{j : \sigma(j) > \sigma(i), j < i \}.
\end{equation}

To quantify dissimilarity, the Kendall tau distance between two rankings $\sigma_1$ and $\sigma_2$ is the number of inversions in $\sigma_1^{-1} \circ \sigma_2$:\vspace{-2mm}
\begin{equation}
    d_{\kappa}\left(\sigma_1, \sigma_2\right) := \sum_{i=1}^n \inv(\sigma_1^{-1} \circ \sigma_2)_i.
\end{equation}
In other words, it is the number of pairwise disagreements, or \textit{discordant} pairs, in the permutation ordering.
The distance is one affine transform away from the Kendall tau correlation, used to measure list order similarity~\cite{kendall1948rank}:\vspace{-1mm}
\begin{equation}
\tau(\sigma_1, \sigma_2) := 1 - \frac{2d_{\kappa}(\sigma_1, \sigma_2)}{\binom{n}{2}}.
\end{equation}
In the extreme, $\tau=1 \iff \sigma_1 = \sigma_2$, and $\tau=-1$ implies that one is the other's reverse.

\subsection{Permutation Self-Consistency}\label{sec:psc}
How do we mitigate positional biases in listwise-ranking LLMs?
We find inspiration in the self-consistency framework~\cite{wang2023selfconsistency}, which improves quality and consistency in chain-of-thought prompting~\cite{wei2022chain}.
The approach has two main stages:\ first, it \textit{samples} multiple answers for an input prompt;\ then, it \textit{aggregates} the sampled answers into a single, high-quality one, hence ``marginalizing out'' separate reasoning paths from the language model.

Unfortunately, self-consistency does not readily generalize to listwise ranking for a few reasons.
For one, it is limited to point predictions, greatly simplifying the aggregation procedure to taking the majority vote.
For another, sampling temperature, the method's mainstay of generating diverse samples for aggregation, has little effect on (and at times harming) the quality of aggregated predictions in listwise ranking, as shown in Section~\ref{sec:hyperparameter}.
Lastly, self-consistency does not explicitly address positional bias, the central issue of our paper.

Nevertheless, its shuffle--aggregate paradigm is still a useful template.
With it, we propose \textit{permutation self-consistency}:\ for the first sample step, we randomly shuffle the list in the prompt to curate a diverse set of rankings, each with different position biases.
For the next aggregate step, we compute the central ranking closest in Kendall tau distance to all the sampled rankings, which, like self-consistency, marginalizes out the independent variable (in the original, reasoning paths;\ in ours, prompt order).
Intuitively, we intervene on list order, collect output rankings, then aggregate, breaking the association between individual list order and output rankings.

\begin{table}[t]\small
    \setlength{\tabcolsep}{1.75pt}
    \centering
    \begin{tabular}{ll}
    \toprule[1pt]
    \textbf{Task} & \textbf{Example Input Prompt} \\
    \midrule[1pt]
    Math Sorting & Sort these expressions: 3 / 2, 1 - 5, ... \\
    \midrule
    Sentence Ordering & Order the shuffled sentences: [1] The...\\
    \midrule
    Passage Ranking & Order these by relevance to the query, \\
    & ``what are shrews?'': [1] Cats hunt...\\
    \bottomrule[1pt]
    \end{tabular}
    \caption{Listwise-ranking input prompt examples.}
    \label{tab:ex3}
\end{table}

Formally, we are given an input sequence of items $\bs X := \{X_i\}_{i=1}^n$, such as a list of passages, along with a listwise-ranking LLM $h(\bs X;\ s)$ that returns an $n$-ranking on some string prompt $s$;\ see Table~\ref{tab:ex3} for an example.
First, we construct a diverse set of output rankings by randomly permuting $\bs X$ and passing it through the LLM, like how self-consistency uses temperature to vary their output.
Specifically, we sample a sequence
\begin{equation}\label{eqn:shuffle}
    \hat{\sigma}_i := h(\bs X[\pi_i];\ s)\text{ for }1 \leq i \leq m,
\end{equation}
where $\pi_i$ is drawn uniformly at random from the set of all possible $n$-rankings.
As noted previously, each output ranking has positional bias, but mistakes are expected to differ among the outputs because of our input order randomization.
We then ``marginalize out'' these individual biases by aggregating the output rankings into a single central ranking.
One method with attractive theoretical properties is the Kemeny--Young~\cite{kemeny1959mathematics} optimal ranking of the outputs---that is, the central ranking that minimizes the sum of its Kendall tau distances to every output ranking:
\begin{equation}
    \bar{\sigma} := \argmin_{\sigma}\sum_{1 \leq i \leq m} d_\kappa(\hat{\sigma}_i, \sigma).
\end{equation}
Our approach returns $\bar{\sigma}$ as the prediction for $\bs X$ and terminates.
Although this calculation is NP-hard, fast exact and approximate algorithms exist~\cite{conitzer2006improved, ali2012experiments}, many implemented in our codebase.

\parheader{Passage reranking}
The task of passage ranking is to rank a set of provided passages in order of relevance to a given query.
The use of permutation self-consistency for this case deserves special attention.
Due to the LLM input length constraint, predominant LLM-based approaches such as RankGPT~\cite{sun2023chatgpt}, LRL~\cite{ma2023zero}, and RankVicuna~\cite{pradeep2023rankvicuna} stride the LLM across fixed windows of items from the back of the list to the front, rather than output a ranking in a single pass.
In this case, we apply permutation self-consistency to each window.

\subsection{Theoretical Guarantees}\label{sec:our-approach:theoretical-guarantees}
We now show that for certain kinds of noisy rankings, the Kemeny ranking can recover the true ranking given enough observations.
For example, if there always exists some random pair of items that is correctly ranked among randomly ordered observations, we will converge to the true ranking.

\begin{definition}
    For two rankings $\sigma_1$ and $\sigma_2$, the concordant subset is a set $S'$ where $\forall i$ and $ j \in S', \sigma_1(i) < \sigma_1(j) \wedge \sigma_2(i) < \sigma_2(j)$ or $\sigma_1(i) > \sigma_1(j) \wedge \sigma_2(i) > \sigma_2(j)$. 
\end{definition}

\begin{proposition}\label{prop:denoise1}
Let there be a true ranking $\sigma$ and a sequence of i.i.d.\ uniformly noisy rankings $\hat{\bs\sigma} := \{\hat{\sigma}_i\}_{i=1}^m$.
Suppose each noisy ranking $\hat{\sigma}_k$ has a uniformly random, nonempty concordant subset $S'_k$ with $\sigma$, and the remaining rank elements not in $S'_k$ represent a random permutation.
Then the Kemeny--Young ranking $\bar{\sigma}$ of $\hat{\bs\sigma}$ converges in probability to $\sigma$, i.e., it is a consistent estimator.
\end{proposition}
\begin{proofsketch}
Let $A_{ij}$ be the event that the sum of discordant pairs indexed by $i$ and $j$ between $\hat{\bs\sigma}$ and $\sigma$ is greater than the number of concordant ones.
$\pee(A_{ij})$ is upper-bounded by $\exp({-O(m)})$.
The union bound of $\pee(\bigcap_{i,j} A_{ij})$ shows that the probability of the sum of discordant pairs being greater than that of the concordant pairs vanishes for any pair as $m$ approaches infinity.
Thus, the Kemeny-optimal ranking will always approach $\sigma$ for $m \to \infty$, concluding our proof.
\end{proofsketch}

\noindent To extend this, we prove that, in the presence of ranking noise, characterized empirically in Section~\ref{sec:position-bias}, our approach yields a consistent estimator for the true ranking, given that at least one possibly nonrandom pair of items is always concordant:
\begin{proposition}
Let there be a true ranking $\sigma$ and a distribution of noisy rankings $\pee(\sigma_\text{noise})$, where $\sigma_\text{noise}\circ\pi$ always has a uniform, non-empty concordant subset $S$ with $\sigma$ for any input ranking $\pi$, and the elements not in $S$ are uniformly random.
Then the permutation self-consistency procedure is a consistent estimator of $\sigma$ when applied to the input $\pi$ and the ``LLM'' characterized by $\pee(\sigma_\text{noise})$.
\end{proposition}
\begin{proofsketch}
    Observe that the first shuffling stage of permutation self-consistency transforms the premises into those of Proposition~\ref{prop:denoise1}.
    Since the next stage of the method involves the same Kemeny--Young ranking as the proposition does, the rest of the proof quickly follows.
\end{proofsketch}
\noindent Full proofs are in Appendix~\ref{sec:proofs}.

\begin{table}[t]\small
    \setlength{\tabcolsep}{1.75pt}
    \centering
    \begin{tabular}{p{2.8in}}
    \toprule[1pt]
    \textbf{1. MathSort}: Sort ten arithmetic expressions by value. \\
    \midrule[0.1pt]
    \textit{Example}: \texttt{3 / 5}, \texttt{2 - 9}, \texttt{6 * 5}, \texttt{2 * 1}, \texttt{3 / 1}, \texttt{9 * 9}, \texttt{1 - 9}, \texttt{9 + 8}, \texttt{3 / 5}, \texttt{1 / 9}.\vspace{1mm}\\
    \midrule[1pt]
    \textbf{2. WordSort}: Order ten words alphabetically. \\
    \midrule[0.1pt]
    \textit{Example}: \texttt{aaron}, \texttt{roam}, \texttt{aardvark}, \texttt{nexus}, [...].\vspace{1mm}\\
    \midrule[1pt]
    \textbf{3. GSM8KSort}: Unscramble sentences from GSM8K. \\
    \midrule[0.1pt]
    \textit{Example}: Order the scrambled sentences logically:\\
    - \texttt{She took 1 hour to walk the first 4 miles} [...]\\
    - \texttt{Marissa is hiking a 12-mile trail.}\\
    - \texttt{If she wants her average speed to be 4} [...]\\
    \bottomrule[1pt]
    \end{tabular}
    \caption{Example prompts for our three sorting tasks.}
    \label{tab:ex4}
\end{table}

\section{Experiments}
We experiment on sorting and passage ranking, two distinct types of problems in listwise ranking.

\subsection{Sorting Tasks}

\parheader{Setup}
We build three functionally distinct datasets called MathSort, WordSort, and GSM8KSort, corresponding to numerical sorting, alphabetical ordering, and sentence arrangement, respectively.
For MathSort, the task is to sort ten random mathematical expressions of the form \texttt{digit op digit}, where \texttt{digit} is a single digit and 
\texttt{op} is one of \texttt{+}, \texttt{-}, \texttt{*}, or \texttt{/}.
In WordSort, the goal is to order ten random English words alphabetically.
Finally, GSM8KSort is a sentence-unscrambling task over the test set of the GSM8K reasoning dataset~\cite{cobbe2021training}.
For consistency and tractability, we use 100 examples in each dataset;\ see Table~\ref{tab:ex4} for prompts.

These synthetic sorting datasets have certain benefits.
The items are intrinsically comparable, especially in MathSort and WordSort, whose elements have unequivocal order (e.g., ``aardvark'' must precede ``abacus'' in WordSort).
On the other hand, passage ranking relies on human judgment, where label noise may confound findings.
Synthetic construction also enables control of item length:\ MathSort examples are fixed at three tokens, WordSort at a single word, and GSM8K one sentence.

For our LLMs, we choose the open families of LLaMA v2 models~\cite{touvron2023llama}, Mistral-7B Instruct~\cite{jiang2023mistral}, and Zephyr$_\beta$-7B~\cite{tunstall2023zephyr}, along with the closed GPT-3.5 (Turbo, the ``0613'' version) and GPT-4 from OpenAI, both the state of the art.
We apply permutation self-consistency with $m=20$ output rankings, resulting in 20 parallel calls to the LLM per example.
Detailed settings are in Appendix~\ref{sec:appendix:sorting-tasks}.

\begin{table}[t]\small
    \setlength{\tabcolsep}{1.5pt}
    \centering
    \begin{tabular}{lcccccccc}
    \toprule[1pt]
    \multirow{2}{*}{Method} & \multicolumn{2}{c}{\textsc{MathSort}} & \multicolumn{2}{c}{\textsc{WordSort}} & \multicolumn{2}{c}{\textsc{GSM8KSort}} \\
    \cmidrule(lr){2-3} \cmidrule(lr){4-5} \cmidrule(lr){6-7}
    & Orig. & PSC & Orig. & PSC  & Orig. & PSC  \\
    \midrule[1pt]
    Mistral-7B & 34.7 & \underline{52.9} & 55.3 & \underline{74.2} & 46.7 & \underline{65.3} \\
    Zephyr$_\beta$-7B & 13.2 & \underline{32.2} & 30.7 & \underline{60.8} & 34.5 & \underline{61.6} \\
    LLaMA$_2$-7B & 8.7 & \underline{24.2} & 41.3 & \underline{59.9} & 6.1 & \underline{21.3} \\
    LLaMA$_2$-13B & 16.7 & \underline{26.0} & 65.4 & \underline{78.8} & 42.7 & \underline{46.8} \\
    LLaMA$_2$-70B & 27.9 & \underline{31.3} & 74.6 & \underline{81.0} & 61.1 & \underline{71.2} \\
    GPT-3.5 & 64.0 & \underline{75.2} & 85.9 & \underline{88.1} & 82.1 & \underline{88.4}\\
    GPT-4 & 83.5 & \underline{\textbf{89.6}} & 89.9 & \underline{\textbf{92.0}} & 88.4 & \underline{\textbf{90.5}}\\
    \bottomrule[1pt]
    \end{tabular}
    \caption{Kendall tau correlation scores on our sorting tasks. Original scores are the median across 20 single runs, and PSC aggregates those 20. Underline indicates improvement from PSC and bold denotes best.}
    \label{tab:results-sort}
\end{table}

\begin{figure}[t]
    \centering
    \hspace{-2.35mm}\includegraphics[scale=0.355,trim={0.3cm 0.8cm 0.5cm 0.5cm},clip]{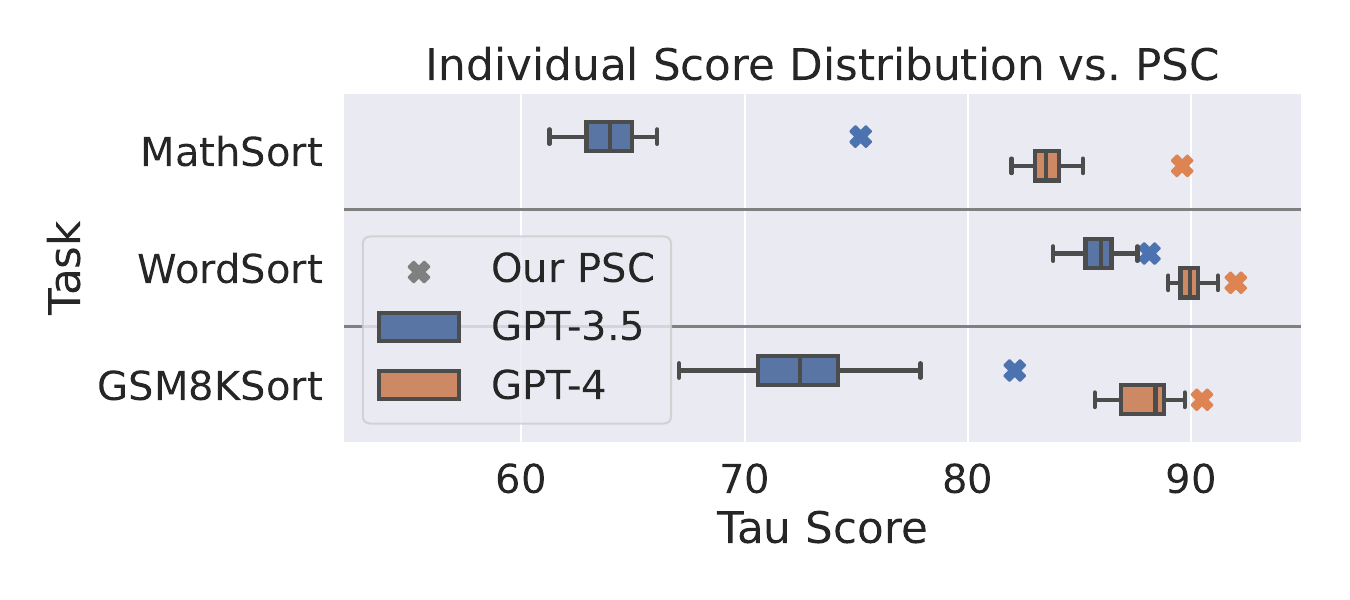}
    \caption{The distribution of sorting task scores from twenty individual runs plotted against our PSC score. Our PSC outperforms the best of any individual run.}\label{fig:sorting-results}
\end{figure}

\begin{table*}[t]\small
    \centering
    \begin{tabular}{lllcccc}
    \toprule[1pt]
    \multirow{2.3}{*}{First Stage} & \multirow{2.3}{*}{Top-$k$} &  \multirow{2.3}{*}{Method} & \multicolumn{2}{c}{\bf TREC-DL19} & \multicolumn{2}{c}{\bf TREC-DL20} \\[-0.5ex]
     \cmidrule(lr){4-5} \cmidrule(lr){6-7}
    & & & Original & Our PSC & Original & Our PSC \\
    \midrule[1pt]
    None & All & (1) BM25 & 50.58 & -- & 47.96 & -- \\
    & All & (2) SPLADE++ ED & 73.08 & -- & 71.97 & -- \\
    \midrule\\[-3ex]
    \multicolumn{7}{c}{Supervised Approaches}\\[-0.5ex]
    \midrule
    BM25 & 100 & (3) MonoT5 (T5-3B) & 71.83 & -- & 68.89 & -- \\
    & 100 & (4) RankT5 (T5-3B) & 71.22 & -- & 69.49 & -- \\
    & 100 & (5) RankLLaMA (13B) & 73.22 & -- & 70.38 & --\\
    \midrule\\[-3ex]
    \multicolumn{7}{c}{Unsupervised Approaches}\\[-0.5ex]
    \midrule
    BM25 & 100 & (6) PRP-Best (FLAN-T5-XXL) & 69.87 & -- & 69.85 & --\\
    & 100 & (7) PRP-Best (FLAN-UL2) & 72.65 & -- & 70.68 & --\\
    &  100 & (8) RankVicuna & 66.83 & \underline{68.70} & 65.49 & \underline{65.68} \\
    &  20 & (9) Single (GPT-3.5) & 60.95 (60.96) & \underline{61.49} & 57.64 (57.68) & \underline{59.62}  \\
    &  20 & (10) Single (GPT-4) &  60.88 (60.92) & \underline{64.88} & 57.78 (57.89) & \underline{62.49}  \\
    &  100 & (11) RankGPT (GPT-3.5) & 68.00 (68.13) & \underline{70.77} & 62.08 (63.20) & 62.70 \\
    & 100 & (12) RankGPT (GPT-4) & 75.00 (75.59) & \underline{\textbf{75.66}} & 70.36 (70.56) & \underline{\textbf{71.00}} \\
    \midrule
    SPLADE++ ED & 100 & (13) RankVicuna & 74.59 & 74.13 & 74.73 & 74.06 \\
    &  20 & (14) Single (GPT-4) & 73.21 (73.36) & \underline{\textbf{76.87}} & 71.97 (73.63) & \underline{\textbf{78.52}}\\
    & 100 & (15) RankGPT (GPT-4) & 74.64 (74.93) & \underline{76.01} & 70.76 (71.08) & \underline{75.14} \\
    \bottomrule[1pt]
    \end{tabular}
    \caption{nDCG@10 results on DL19 and 20. The maximum across three runs are in parentheses, while those outside the median. Improvements from PSC are underlined and best per section are bolded. On the one-tailed signed-rank test, paired differences between the original and PSC are significant at the 99\% confidence level ($p<0.01$).}
    \label{tab:results-ir}
\end{table*}

\parheader{Results}
We present our main results in Table~\ref{tab:results-sort}, naming our method ``PSC'' for short.
PSC consistently outperforms conventional inference on all three datasets and seven models by an average of 51\% in Kendall tau correlation, skewed toward the smaller variants.
Specifically, LLaMA$_2$-7B, 13B, and 70B attain average score increases of 157\%, 28\%, and 12\%, respectively, Mistral and Zephyr improve by 42\% and 106\%, and GPT-3.5 and GPT-4 by 3--18\% and 2--7\%.
We attribute this to the already high quality of the larger 70B and GPT models, which leave less room for improvement.
Task-wise, we improve MathSort, WordSort, and GSM8KSort by 67\%, 30\%, and 58\%, and gains negatively correlate with original quality ($r=-0.72$).
We conclude that PSC improves listwise ranking on sorting tasks, with higher gains on smaller models and more difficult tasks.

One foreseeable question is whether any individual runs surpass PSC, which would weaken the case for rank aggregation.
To answer this, we plot the distribution of the individual scores against PSC in Figure~\ref{fig:sorting-results}.
We observe that PSC reliably beats all individual runs by 1--12\%, improving the most on tasks and models with lower baseline quality, such as MathSort and GPT-3.5.
These findings bolster the necessity of the aggregation step.

\subsection{Passage Reranking Task}
For a longer-context task, we evaluate our method on passage reranking.
For a query and an initial list of relevant documents from a fast, first-stage retriever, we must reorder the documents so that more relevant ones come first.

\parheader{Setup}
We select the passage retrieval test sets from the TREC Deep Learning Tracks DL19 and DL20~(\citealp{craswell2020overview}, \citeyear{craswell2021overview}), both canon in the literature~\cite{qin2023large}.
These datasets are built on the MS MARCO v1 corpus~\cite{bajaj2016ms}, which contains 8.8 million passages.
As is standard, we rerank the top-100 passages retrieved by the first-stage BM25~\cite{robertson2009probabilistic} or SPLADE++ EnsembleDistill (ED;\ \citealp{formal2021splade}), reporting nDCG@10 scores for quality.

Like sorting, we pick an open LLM, RankVicuna~\cite{pradeep2023rankvicuna}, fine-tuned from Vicuna~\cite{vicuna2023}, and a closed family, GPT-3.5 and GPT-4---all models match state of the art.
RankVicuna and GPT-3.5 have context lengths of 4096, half of GPT-4's 8192.
We similarly apply permutation self-consistency with $m=20$ runs.
Furthermore, for three of our variants named ``single,'' we reduce the top-100 to 20 and discard the windowing strategy used in RankGPT and RankVicuna, described in Section~\ref{sec:psc}.
This allows us to fit all passages in a single call and thus remove potentially confounding interactions between the windowing method and permutation self-consistency.

\begin{figure*}[t]
    \centering
    \hfill
    \begin{subfigure}[t]{\columnwidth}
        \includegraphics[width=1.03\columnwidth,trim={0.0cm 0cm 20.5cm 0cm},clip]{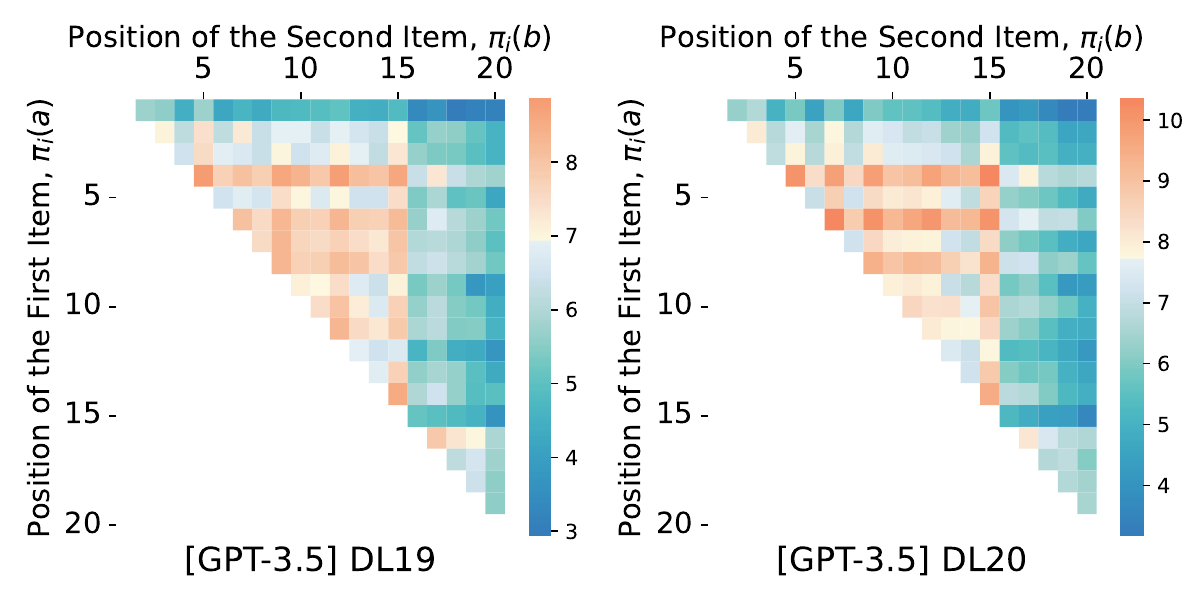}
         \caption{Single (GPT-3.5) on DL19 and DL20.}
         \label{fig:order-bias-a}
    \end{subfigure}
    \hfill
    \begin{subfigure}[t]{\columnwidth}
        \includegraphics[width=1.03\columnwidth,trim={20.5cm 0cm 0cm 0cm},clip]{assets/heatmap.output-order.flipped.blue.pdf}
         \caption{Single (GPT-4) on DL19 and DL20.}
         \label{fig:order-bias-b}
    \end{subfigure}
    \hfill
     
    \caption{
    Distribution of ``reversions'' after reranking. Blues are below the observed dataset average and reds above the average. For two input list positions $i \in [1, 20]$ and $j \in (i, 20]$, $i$ indexes the rows and $j$ the columns. For example, the cell at $(1, 2)$ is the reversion of the first two input items across the dataset.
    Note that highly saturated colors indicate over- and under-reversion \textit{relative} to other pairs in the dataset rather than in the absolute sense.
    }\label{fig:order-bias}
\end{figure*}

For our supervised baselines, we report results from the MonoT5~\cite{nogueira2020document} and RankT5~\cite{zhuang2023rankt5} models, based on the T5 language model~\cite{raffel2020exploring}.
We also run RankLLaMA~\cite{ma2023fine}, the current pointwise state of the art.
For the unsupervised baselines, we copy figures from the state-of-the-art pairwise ranking results across the variants in \citet{qin2023large}, which we name PRP-Best for short.

\parheader{Results}
We present our results in Table~\ref{tab:results-ir}.
Our PSC outperforms all conventional inference baselines:\ first, RankGPT with PSC on DL19 (row 12) edges ahead by 0.07 points (same row);\ second, the same for DL20 (row 12), leading PRP by 0.32 points (row 7);\ third, the overall top result on DL19 of 76.87 from SPLADE++ (row 14), outperforming the previous by 1.28 (row 12);\ and fourth, 78.52 on DL20 (row 14), a 3.79-point increase over RankVicuna (row 13), the best single-call baseline model.
For qualitative examples, see Appendix~\ref{sec:appendix:qualitative-examples}.

Overall, our PSC approach consistently improves ordinary decoding and beats the maximum individual score across three runs (see scores in parentheses), yielding gains on 13 out of 16 model--dataset combinations (see PSC columns in rows 7--14).
On average, RankVicuna, GPT-3.5, and GPT-4 see relative score increases of 0.4\%, 2\%, and 5\% with PSC.
Mixed results on RankVicuna likely result from its inherent robustness to positional bias, instilled by its training process that uses random shuffling as part of data augmentation;\ thus, the shuffling step from PSC has less of an effect on the output variation.

The choice of the first-stage reranker has a clear impact, with SPLADE++ adding an average of 7.26 points over the corresponding BM25 models.
In fact, reranking the top-20 SPLADE items (row 13) in a single call outperforms doing the top-100 (row 14) using a sliding call window.
We conjecture that this results from imperfections in the RankGPT windowing algorithm, which shows especially for strong retrievers, where the top-20 already contains many relevant documents.

Finally, we note one particularly intriguing phenomenon:\ in the top-20 single-call setting, GPT-3.5 and GPT-4 have similar baseline quality without PSC (rows 8 and 9, first column in each group), but PSC boosts GPT-4 more than GPT-3.5 (row 9, second columns).
As we explore in depth next, this possibly results from GPT-4 being more ``equally biased'' across the item positions and hence providing PSC more useful rankings for aggregation.

\parheader{Positional bias analysis}\label{sec:position-bias}
We analyze how list order bias varies with the input positions on the ``single'' GPT models for BM25 (from \autoref{tab:results-sort}, rows 8 and 9), which avoids confounds from RankGPT's window strategy.
The design of our analysis is as follows, mirroring Section~\ref{sec:psc}'s notation:\ consider the item pair $(X_a, X_b)$ with input list positions $(\pi_i(a), \pi_i(b))$, where $\pi_i(a) < \pi_i(b)$ for some random permutation $\pi_i$.
If the output positions satisfy $\hat{\sigma}_i(a) > \hat{\sigma}_i(b)$ after reranking,
we say the order is reversed, and we call the sum of reversed pairs per data point \textit{``reversions.''}
In \autoref{fig:order-bias}, we visualize the distribution of reversions by input position pair, with $\pi_i(a)$ as the $y$-axis and $\pi_i(b)$ as the $x$-axis, whose positions range from 1--20 for each of the top-20 passages.
For cross-model comparability, we normalize by dataset.

Under the null hypothesis of there being no positional bias, the distribution of reversions should be uniform because the input lists are randomly permuted, which severs any association between input order and output ranking.
However, \autoref{fig:order-bias} contradicts this.
Prominently, the center of \autoref{fig:order-bias-a} is redder than the edges,
indicating that pairs with both items closer to the middle are reversed more often by GPT-3.5 than those at the beginning and the end of the input lists are.
In \autoref{fig:order-bias-b}, bottom areas are also deeper red than the top, 
showing that pairs with items at the end of the list are more frequently reversed by GPT-4 than pairs at the start.

\begin{figure*}[t]
    \centering
    \hfill
    \begin{subfigure}[t]{\columnwidth}
        \centering
        \hspace{-2mm}\includegraphics[width=\columnwidth,trim={0.1cm 0.25cm 0.35cm 0cm},clip]{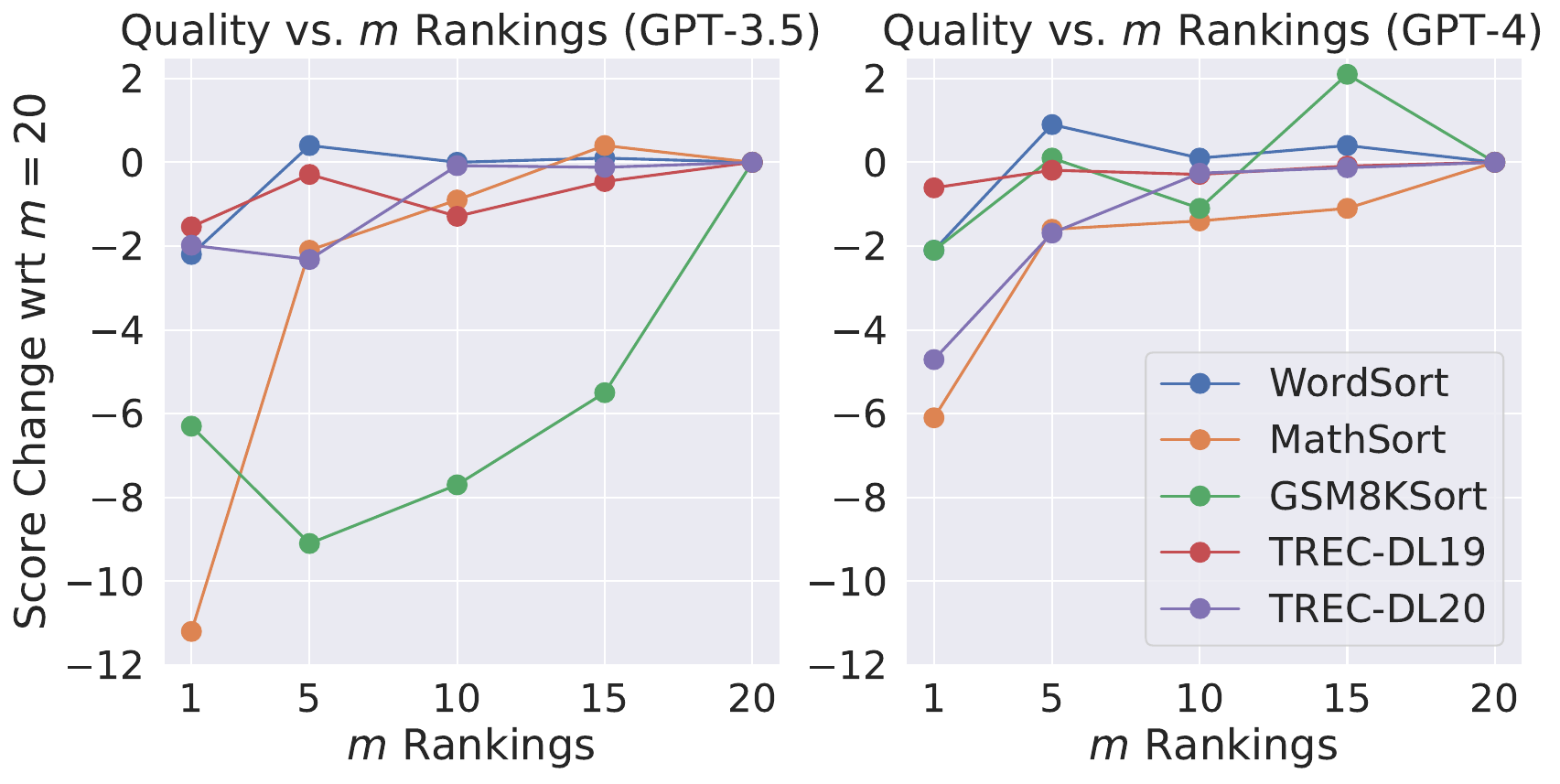}
        \caption{Quality vs.\ number of output rankings ($\rho=0.17$).}\label{fig:shuffle-results}
    \end{subfigure}
    \hfill
    \begin{subfigure}[t]{\columnwidth}
        \centering
        \hspace{-2mm}\includegraphics[width=\columnwidth,trim={0.1cm 0.25cm 0.35cm 0cm},clip]{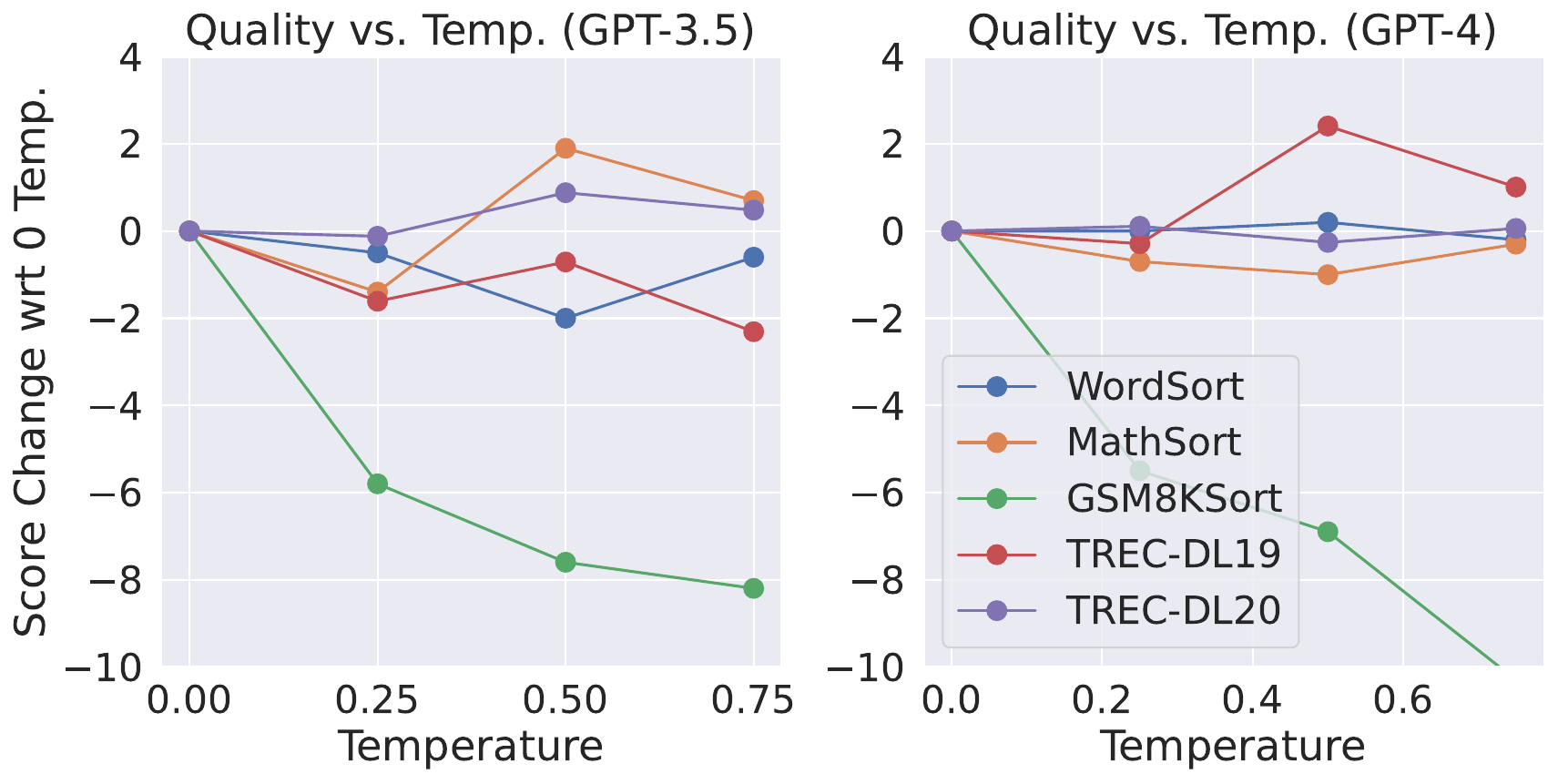}
        \caption{Quality vs.\ text generation temperature ($\rho=-0.078$).}\label{fig:temperature-results}
    \end{subfigure}
     \hfill
    \caption{Quality for all datasets for various aggregate sizes and temperatures. For output rankings, we use $m=20$ as our frame of reference;\ for temperature, $0.0$. In the subfigure captions, $\rho$ denotes Spearman's rank correlation.}\label{fig:hyperparameters}
\end{figure*}

Other subtle patterns emerge upon closer examination.
First, in \autoref{fig:order-bias-a}, a dark block appears after column 15, suggesting that GPT-3.5 does \textit{not} focus well on items past the fifteenth.
Second, the colors interleave in a grid pattern across both columns and rows---possibly an artifact of its pretraining.
From this evidence, we conclude that different positional biases exist in reranking LLMs, varying by model and dataset.

The analysis also helps to explain our quality results.
Comparing \autoref{fig:order-bias-a}~and~\ref{fig:order-bias-b},
we observe that \mbox{GPT-4} generally reverses more pairs than GPT-3.5 and is closer to the optimal number of reversals, thus providing higher quality to the aggregated rankings.
This may explain why PSC benefits GPT-4 (single) more than it does GPT-3.5 (single), i.e.\ row~9 vs.\ row~8 in \autoref{tab:results-ir}.
Similarly, both models tend to reverse more pairs on DL20 than on DL19, and results also indicate that PSC improves DL20 more than it does DL19.

\section{Sensitivity Analyses}
In this section, we investigate and characterize each component of permutation self-consistency to justify our modeling choices.

\subsection{Hyperparameter Studies}\label{sec:hyperparameter}
\parheader{Output rankings}
Throughout the paper, we espoused aggregating over $m=20$ output rankings, but is more actually better?
If, say, five outperformed twenty, we could decrease the number of parallel calls to the model, conceivably saving cost.
To answer this question, we sweep the aggregate size between one and twenty across all datasets, plotting the resulting score differences from using the default twenty.
We pick GPT-3.5 and GPT-4 as our target models, as they are used in all tasks.

We plot our results in Figure~\ref{fig:shuffle-results}.
On both models, we find that output quality rapidly converges to that of using the full twenty, five being 67\% as effective on average.
The score averages increase monotonically with the number of rankings ($\rho=0.17$), with GSM8KSort on GPT-3.5 as an outlier (left subplot), possibly because of output variance---the next study on sampling temperature shows that it is highly sensitive to randomness.
We conclude that picking $m=20$ output rankings is effective, though returns sharply diminish after 5--10.

\parheader{Sampling temperature}
Self-consistency~\cite{wang2023selfconsistency} uses temperature as their sampling strategy to produce different outputs to aggregate over, but it is ineffective for us, perhaps because listwise ranking does not admit multiple reasoning paths like chain-of-thought prompting does.
To assess this rigorously, we vary the temperature between 0 and 0.75, following the original method's 0.5--0.7~\cite{wang2023selfconsistency}.
For consistency, we use the same setup from before and fix $m=20$.

We plot our results in Figure~\ref{fig:temperature-results}.
Temperature has little effect on the quality ($\rho=-0.078$), again with GSM8KSort as an outlier, where the extra randomness drastically hurts quality on both models.
This sensitivity to randomness is also evident in Figure~\ref{fig:sorting-results}, where GSM8K has the widest interquartile range of the tasks.
In conclusion, this evidence grounds our choice of not using temperature.

\subsection{Rank Aggregation Comparison}
\begin{figure}[t]
    \centering
    \hspace{-2mm}\includegraphics[width=0.93\columnwidth,trim={0.1cm 0.25cm 0.35cm 0cm},clip]{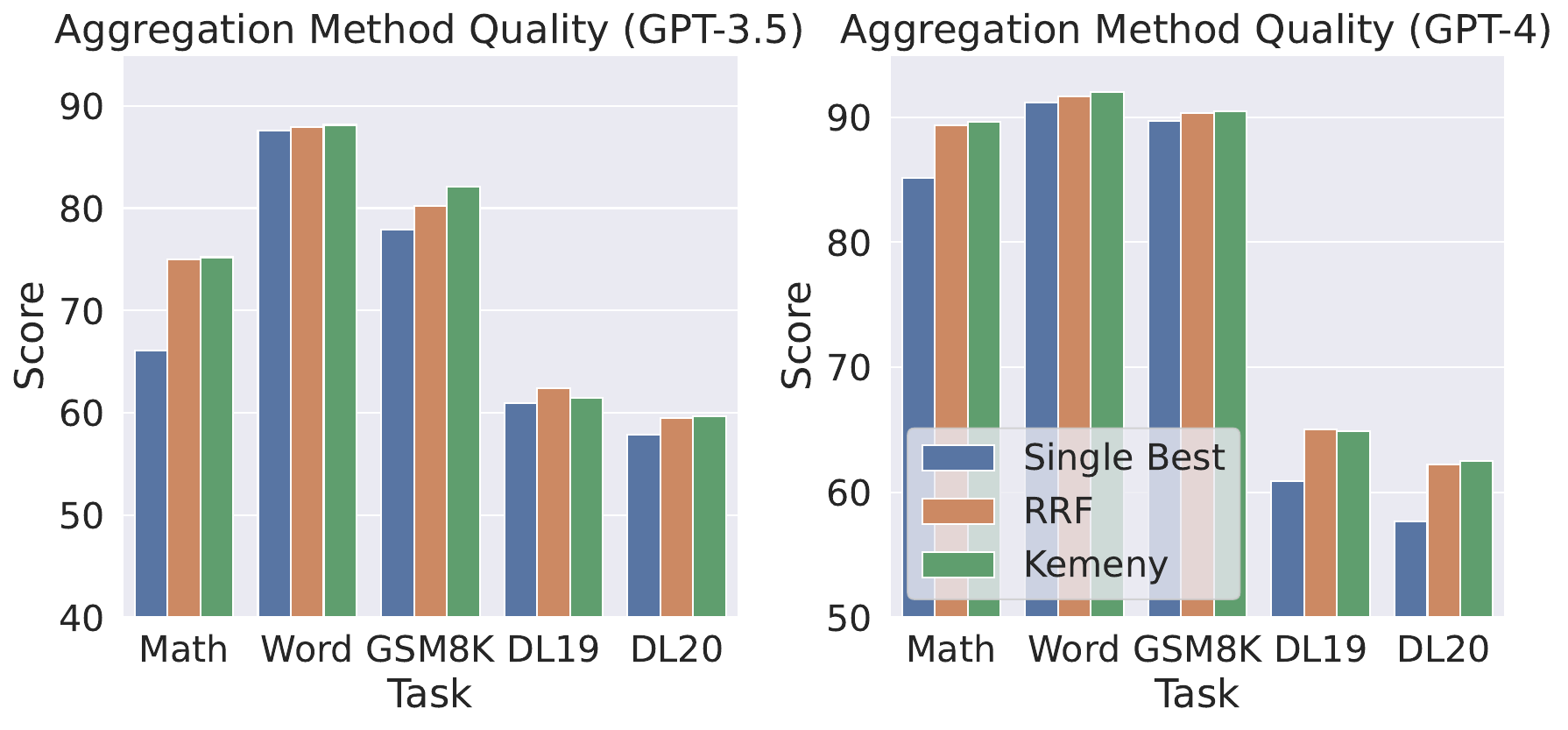}
    \caption{Scores for the alternative reciprocal rank fusion (RRF) and our Kemeny rank aggregation method.}\label{fig:aggregation-results}
\end{figure}
Reciprocal rank fusion (RRF;~\citealp{cormack2009reciprocal}) is a state-of-the-art alternative to our chosen Kemeny ranking method.
It sorts items by the score
\begin{equation}
    \text{RRFScore}(X_j) := \sum_{1 \leq i \leq m} \frac{1}{k + \hat{\sigma}_i(j)}
\end{equation}
for each item $X_j$, rankings $\hat{\sigma}_i$, and $k=60$.
RRF had been under our consideration, but we picked Kemeny ranking for its theoretical robustness and empirical effectiveness.
Shown in Figure~\ref{fig:aggregation-results}, Kemeny beats RRF ($p<0.05$) on 8 out of 10 comparisons by a mean of 0.23 points;\ on average, RRF reaches only 93.5\% of the boost that Kemeny does.
Its only outperformance on DL19 possibly results from it being suited for information retrieval, its field of origin, but this may also be statistical noise.
Overall, these results further support our decision to select Kemeny ranking for the aggregation step.

\section{Related Work and Future Directions}
The holistic direction of our work is in enhancing the ranking ability of large language models.
Along a similar vein, contrast-consistent ranking~\cite{stoehr2023unsupervised} proposes to train order-unaware probes on the latent vectors of large language models for detecting nondirectional rank consistency.
Their evaluation reveals the ranking direction of test examples to the models, deviating from standard practices, as their purpose is not to increase ranking quality but rather to detect consistency.
Another related work is \citet{hou2024large}, which uses a different rank aggregation algorithm from ours.
In contrast to their heuristic bootstrapping method (i.e., Borda count) of summing up the ranks of each ranking, our approach is theoretically optimal in that it finds the best central ranking to all individual rankings in terms of the tau distance.

The specific empirical tasks in this paper have also seen recent progress.
For passage ranking using language models, BERT-based~\cite{devlin2019bert, nogueira2020document} and T5-tuned~\cite{zhuang2023rankt5, raffel2020exploring} approaches represent the earliest language models for passage ranking.
RankGPT~\cite{sun2023chatgpt} and LRL~\cite{ma2023zero} spearheaded much of the post-ChatGPT work, beating the supervised state of the art with an \textit{unsupervised} LLM for the first time.
Along a non-listwise direction, PRP~\cite{qin2023large} is a pairwise method leveraging open-source large language models comparing two items at a time, as reported in Table~\ref{tab:results-ir}.
One possible future work is to reformulate our PSC method to be differentiable, enabling training-time application in LLMs such as RankVicuna~\cite{pradeep2023rankvicuna}.

Our sorting tasks for LLMs have had attention as well, mostly in the context of evaluation, with BigBench~\cite{suzgun2022challenging, srivastava2023beyond}, an LLM benchmark, providing more than 200 distinct tasks, including one in alphabetical ordering (\texttt{word\_sorting}), which we enlarge and expand on in WordSort.
\citet{stoehr2023unsupervised} also constructed fact-based synthetic sorting datasets for listwise ranking, but they are private and hence noncomparable.
In the future, PSC can be applied to any list-oriented ranking task involving LLMs.
Examples include using LLMs for evaluation~\cite{wang2023large} and annotating human feedback judgments with language models.
Additionally, PSC is applicable at training time, such as denoising weakly labeled training sets generated by teacher models, shown to be crucial to the success of listwise-ranking LLMs~\cite{pradeep2023rankvicuna}.

We are not the first to establish positional biases in LLMs.
\citet{lu2022fantastically} are among the earliest to relate prompt order to the quality of in-context learning.
The main difference in setup is that they assume the presence of a training set, whereas we do not, which especially matters for passage ranking, as many tasks only have evaluation sets.
Recently, \citet{liu2023lost} and \citet{wang2023large} characterized positional bias in the context of list-oriented tasks, such as question answering and response evaluation.
However, we are to our knowledge the first to characterize the position biases of passage-ranking LLMs with respect to pairwise item positions, and our work also proposes a correction technique.
Moreover, \citet{pezeshkpour2023large} and \citet{li2023prd} apply prompting-based techniques for mitigating positional bias.
Prompting is not mutually exclusive of our PSC, and it could be complementary.

Lastly, our paper is connected to all the meta-algorithms for improving LLM generation.
As a pertinent example, \citet{lu2022fantastically} study prompt order on in-context learning classification tasks, proposing an entropy-based statistic over development sets to find performant permutations of few-shot examples.
\citet{aggarwal2023let} make self-consistency more efficient, halting the procedure when enough samples have been collected.
To keep our method in its simplest form, as self-consistency had not been applied to listwise ranking to begin with, we based our design on the original approach~\cite{wang2023selfconsistency}.

\section{Conclusions}
We introduce a novel decoding method to improve the ranking ability of black-box LLMs by mitigating potential sensitivities and biases to list item order.
We intervene on prompt list order to produce multiple rankings then return an aggregated statistic as the prediction, which has less association with list order.
Theoretically, we prove the robustness of our method to arbitrary, fixed noise distributions.
Empirically, our method consistently improves upon ordinary decoding on all 15 of our sorting model--dataset combinations and 13 out of 16 of our passage reranking ones.
Finally, our sensitivity analyses justify our design choices of 20 output rankings, zero sampling temperature, and the Kemeny ranking method.

\section*{Limitations}
We share limitations with those of the original self-consistency paper~\cite{wang2023selfconsistency}.
We use multiple LLM calls, potentially to a commercial LLM, which would raise financial cost.
Thus, practical applications may require careful weighing of quality gain against elevated expense.
Nevertheless, a few calls already help, and returns rapidly diminish past 5--10 calls.
We note that our method does not in practice increase latency by much, since all calls can be parallelized, and aggregation time does not rise with the number of samples.
For further discussion, see Appendix~\ref{sec:appendix:computational-burden}.

Another limitation is that GPT-3.5 and GPT-4 are proprietary models lacking official documentation of its internals.
We acknowledge that this is an ongoing issue in the natural language processing literature as of 2023, with many publications relying on the continued existence of these endpoints.
To partially alleviate this, we have run experiments on the open-source Mistral~\cite{jiang2023mistral}, Zephyr~\cite{tunstall2023zephyr}, LLaMA 2~\cite{touvron2023llama}, and RankVicuna~\cite{pradeep2023rankvicuna} models where possible.

Finally, our study is intentionally restricted to automated evaluation in an academic setting.
Kendall's tau and nDCG@10, while standard metrics in evaluating ranking systems, do not exactly capture human preferences.
It remains to be determined how effective permutation self-consistency is for, say, an in-production web search engine or recommendation system.

\bibliography{anthology}

\appendix
\section{Proofs of Propositions}
\label{sec:proofs}
\begin{proposition}[2.1]
Let there be a true ranking $\sigma$ and a sequence of i.i.d.\ uniformly noisy rankings $\hat{\bs\sigma} := \{\hat{\sigma}_i\}_{i=1}^m$.
Suppose each noisy ranking $\hat{\sigma}_k$ has a uniformly random, nonempty concordant subset $S'_k$ with $\sigma$, and the remaining rank elements not in $S'_k$ represent a random permutation.
Then the Kemeny--Young ranking $\bar{\sigma}$ of $\hat{\bs\sigma}$ converges in probability to $\sigma$, i.e., it is a consistent estimator.
\end{proposition}
\begin{proof}
    Our strategy is to upper-bound the probability that the number of discordant pairs between the predicted rankings and the true ranking is greater than the number of concordant pairs, then show that this upper bound approaches zero with enough samples.
    Since the Kemeny-optimal ranking is defined as the ranking that exactly minimizes the difference between these discordant and concordant pairs, the probability that the Kemeny ranking is equivalent to the true ranking shares this upper bound.
    
    Let $A_{ij}$ be the event that the sum of discordant pairs indexed by $i$ and $j$ between $\bs\hat{\bs\sigma}$ and $\sigma$ is greater than the concordant ones, i.e.,
    \begin{align*}
        &A_{ij} := \sum_{k=1}^m \text{disc}_{ij}(\sigma, \hat{\sigma}_k) > \sum_{k=1}^m \text{conc}_{ij}(\sigma, \hat{\sigma}_k)\\
        &= \sum_{k=1}^m \text{disc}_{ij}(\sigma, \hat{\sigma}_k) > m - \sum_{k=1}^m \text{disc}_{ij}(\sigma, \hat{\sigma}_k)\\
        &= \sum_{k=1}^m \text{disc}_{ij}(\sigma, \hat{\sigma}_k) > \frac{m}{2},
    \end{align*}
    with convenience functions $\text{disc}_{ij}(\sigma_1, \sigma_2) := \mathbb{I}((\sigma_1(i) > \sigma_2(j) \wedge \sigma_1(i) < \sigma_2(j)) \vee (\sigma_1(i) < \sigma_2(j) \wedge \sigma_1(i) > \sigma_2(j)))$ and $\text{conc}_{ij} := 1 - \text{disc}_{ij}(\sigma_1, \sigma_2)$ indicating pair discordance and concordance according to the Kendall tau criterion. The LHS of the event also defines a sum of independent random variables (r.v.), each a Bernoulli distribution
    \begin{equation}
        X_k = \text{disc}_{ij}(\sigma, \hat{\sigma}_k) \sim \text{Bernoulli}(p_k).
    \end{equation}
    This is evident because each summand is a binary outcome, which is independent based on the given premise that $\hat{\bs\sigma}$ and concordant subsets $\{S'_k\}_{k=1}^m$ are each independent.
    The probability of $A_{ij}$ can be equivalently stated as
    \begin{equation}
        \mathbb{P}(A_{ij}) = \mathbb{P}(\sum_{k=1}^m X_k > \frac{m}{2}).
    \end{equation}
    We upper-bound the RHS.
    First, since $S'_k$ is non-empty and $\hat\sigma_k$'s elements not in $S'_k$ are uniformly random (as given by the premise), the chance of drawing a discordant ranking is $p_k \leq \frac{1}{2} - \delta$ for some $\delta > 0$.
    Thus, 
    \begin{align}
        \mathbb{P}(\sum_{k=1}^m X_k > \frac{m}{2}) &\leq \mathbb{P}(\sum_{k=1}^m X > \frac{m}{2})\\
        &=\mathbb{P}(\frac{1}{m}\sum_{k=1}^m X > \frac{1}{2}),
    \end{align}
    where $X \sim \text{Bernoulli}(\frac{1}{2} - \delta)$.
    By Hoeffding's inequality, we have for all $\epsilon > 0$
    \begin{equation*}
        \mathbb{P}\left(\frac{1}{m}\sum_{k=1}^m X - (\frac{1}{2} - \delta) > \epsilon\right) \leq \exp(-2m\epsilon^2).
    \end{equation*}
    Let $\epsilon = \delta$.
    Then
    \begin{align}\small
        &\mathbb{P}\left(\frac{1}{m}\sum_{k=1}^m X - (\frac{1}{2} - \delta) > \delta\right)
        \\&=\mathbb{P}\left(\frac{1}{m}\sum_{k=1}^m X - \frac{1}{2} > 0\right)
        \\&=\mathbb{P}\left(\frac{1}{m}\sum_{k=1}^m X > \frac{1}{2}\right)
        \\&=\mathbb{P}(A_{ij})
        \\&\leq\exp(-2m\delta^2).
    \end{align}
    We now consider the probability of any $A_{ij}$, i.e., the probability that the sum of discordant pairs indexed by \textit{any} $i$ and $j$ between $\hat{\bs\sigma}$ and $\sigma$ is greater than the concordant ones.
    By the union bound,
    \begin{align}
        \mathbb{P}(\bigcup_{i< j}A_{ij}) &\leq \sum_{i< j}\mathbb{P}(A_{ij})\\
        &\leq \binom{n}{2}\exp(-2m\delta^2)\\
        &\leq n^2\exp(-2m\delta^2).
    \end{align}
    Taking $m \to \infty$, the RHS $=0$.
    Since the Kemeny-optimal ranking always chooses the ranking that minimizes pairwise discordance (picking any other ranking would increase the Kendall tau distance, a contradiction with the definition of Kemeny optimality), for $m \to \infty$ we recover the true ranking with probability $1$, completing our proof that it is a consistent estimator.
\end{proof}%

\begin{proposition}[2.2]
Let there be a true ranking $\sigma$ and a distribution of noisy rankings $\pee(\sigma_\text{noise})$, where $\sigma_\text{noise}\circ\pi$ always has a uniform, non-empty concordant subset $S$ with $\sigma$ for any input ranking $\pi$, and the elements not in $S$ are uniformly random.
Then the permutation self-consistency procedure is a consistent estimator of $\sigma$ when applied to the input $\pi$ and the ``LLM'' characterized by $\pee(\sigma_\text{noise})$.
\end{proposition}
\begin{proof}
Our technique is to show that the premises of Proposition 2.2 can be transformed to those of 2.1, which we have a proof for.
Let $\pi$ be drawn uniformly at random from the sample space of all permutations, $\Omega$, as in the first step of the permutation self-consistency procedure.
From the premise of both the concordant subset $S$ of $\sigma_\text{noise}\circ\pi$ and its complement $S^C$ being uniformly random, letting $\hat{\bs\sigma}$ be realizations of $\sigma_\text{noise} \circ \pi$ fulfills the premise for Proposition~\ref{prop:denoise1}.
The rest of our proof follows from that of 2.1.
\end{proof}

\section{Detailed Experimental Setup}
\begin{table}[t]\small
    \setlength{\tabcolsep}{1.75pt}
    \centering
    \begin{tabular}{p{2.8in}}
    \toprule[1pt]
    \textbf{1. MathSort}: Sort ten arithmetic expressions by value. \\
    \midrule[0.1pt]
    \texttt{<User>} Sort the following expressions from smallest to largest: \texttt{3 / 5}, \texttt{2 - 9}, \texttt{6 * 5}, \texttt{2 * 1}, \texttt{3 / 1}, \texttt{9 * 9}, \texttt{1 - 9}, \texttt{9 + 8}, \texttt{3 / 5}, \texttt{1 / 9}. The output format should be a comma-separated list containing the exact expressions;\ do not reduce them. Only respond with the results;\ do not say any word or explain.\vspace{1mm}\\
    \midrule[1pt]
    \textbf{2. WordSort}: Order ten words alphabetically. \\
    \midrule[0.1pt]
    \texttt{<User>} Order these words alphabetically:\ \texttt{aaron}, \texttt{roam}, \texttt{aardvark}, \texttt{nexus}, [...]. The output format should [...]\vspace{1mm}\\
    \midrule[1pt]
    \textbf{3. GSM8KSort}: Unscramble sentences from GSM8K. \\
    \midrule[0.1pt]
    \texttt{<User>} Order the scrambled sentences logically:\\
    - \texttt{She took 1 hour to walk the first 4 miles} [...]\\
    - \texttt{Marissa is hiking a 12-mile trail.}\\
    - \texttt{If she wants her average speed to be 4} [...]\\
    The output format should have each sentence on a new line. Only respond with the results;\ do not say any [...]\\
    \bottomrule[1pt]
    \end{tabular}
    \caption{Full prompts for our three sorting tasks. ``\texttt{<User>}'' is a model-specific prefix token qualifying the subsequent message as belonging to the user for instruction prompting.}
    \label{tab:ex4-detailed}
\end{table}
\subsection{Computational Environment}\label{sec:appendix:computational-environment}
We conducted the experiments on a machine running Ubuntu 22.04 with two Nvidia A6000 GPUs, an AMD Epyc Milan 7B13 CPU, and 256GB of ECC RAM.
Our most relevant software frameworks included PyTorch 2.1.0, Transformers 4.36.1, PuLP 2.7.0, and CUDA 12.2.
Where possible, we used FlashAttention v2~\cite{dao2023flashattention, dao2022flashattention} and BF16 to accelerate the LLMs.

\subsection{Sorting Tasks}\label{sec:appendix:sorting-tasks}
\autoref{tab:ex4-detailed} lists the full prompts used in our sorting tasks.
To extract the rankings, we examined the outputs and wrote regular expressions;\ all the models capably generated well-formed, extractable text, in line with the claims in their papers~\cite{tunstall2023zephyr, jiang2023mistral, touvron2023llama}.
All prompts fit in a context size of 4096 tokens.
\begin{table}[t]\small
    \setlength{\tabcolsep}{1.75pt}
    \centering
    \begin{tabular}{p{2.8in}}
    \toprule[1pt]
    \textbf{RankVicuna}: Prompt from \citet{pradeep2023rankvicuna}. \\
    \midrule[0.1pt]
    \texttt{<User>} I will provide you with \{num\} passages, each indicated by a numerical identifier []. Rank the passages based on their relevance to the search query: \{query\}.\\\\
    
    [1] \{passage 1\}\\
    
    [2] \{passage 2\}\\
    ...\\
    
    [\{num\}] \{passage \{num\}\}\\\\
    Search Query: \{query\}.\\
    Rank the \{num\} passages above based on their
    relevance to the search query. All the passages
    should be included and listed using identifiers, in
    descending order of relevance. The output format
    should be [] > [], e.g., [4] > [2]. Only respond
    with the ranking results, do not say any word
    or explain.\vspace{1mm}\\
    \midrule[1pt]
    \textbf{RankGPT}: Prompt from \citet{sun2023chatgpt}. \\
    \midrule[0.1pt]
    \texttt{<System>} You are RankGPT, an intelligent assistant that can rank passages based on their relevancy to the query.\\[0.5ex]
    \texttt{<User>} I will provide you with \{num\} passages, each indicated by number identifier []. \\Rank the passages based on their relevance to query: \{query\}.\\\\
    
    [1] \{passage 1\}\\
    
    [2] \{passage 2\}\\
    ...\\
    
    [\{num\}] \{passage \{num\}\}\\\\
    Search Query: \{query\}. \\Rank the \{num\} passages above based on their relevance to the search query. The passages should be listed in descending order using identifiers. The most relevant passages should be listed first. The output format should be [] > [], e.g., [1] > [2]. Only response the ranking results, do not say any word or explain.\vspace{1mm}\\
    \bottomrule[1pt]
    \end{tabular}
    \caption{Full prompts for our passage reranking task. ``\texttt{<User>}'' and ``\texttt{<System>}'' are model-specific prefix tokens denoting the user and system roles. Nuances between RankVicuna and RankGPT include grammatical changes and no system prompt.}
    \label{tab:passage-rerank-detailed}
\end{table}

\begin{figure*}[t!]
    \begin{subfigure}[t]{1\columnwidth}
        \includegraphics[width=\columnwidth,trim={0cm 1.15cm 0cm 0cm},clip]{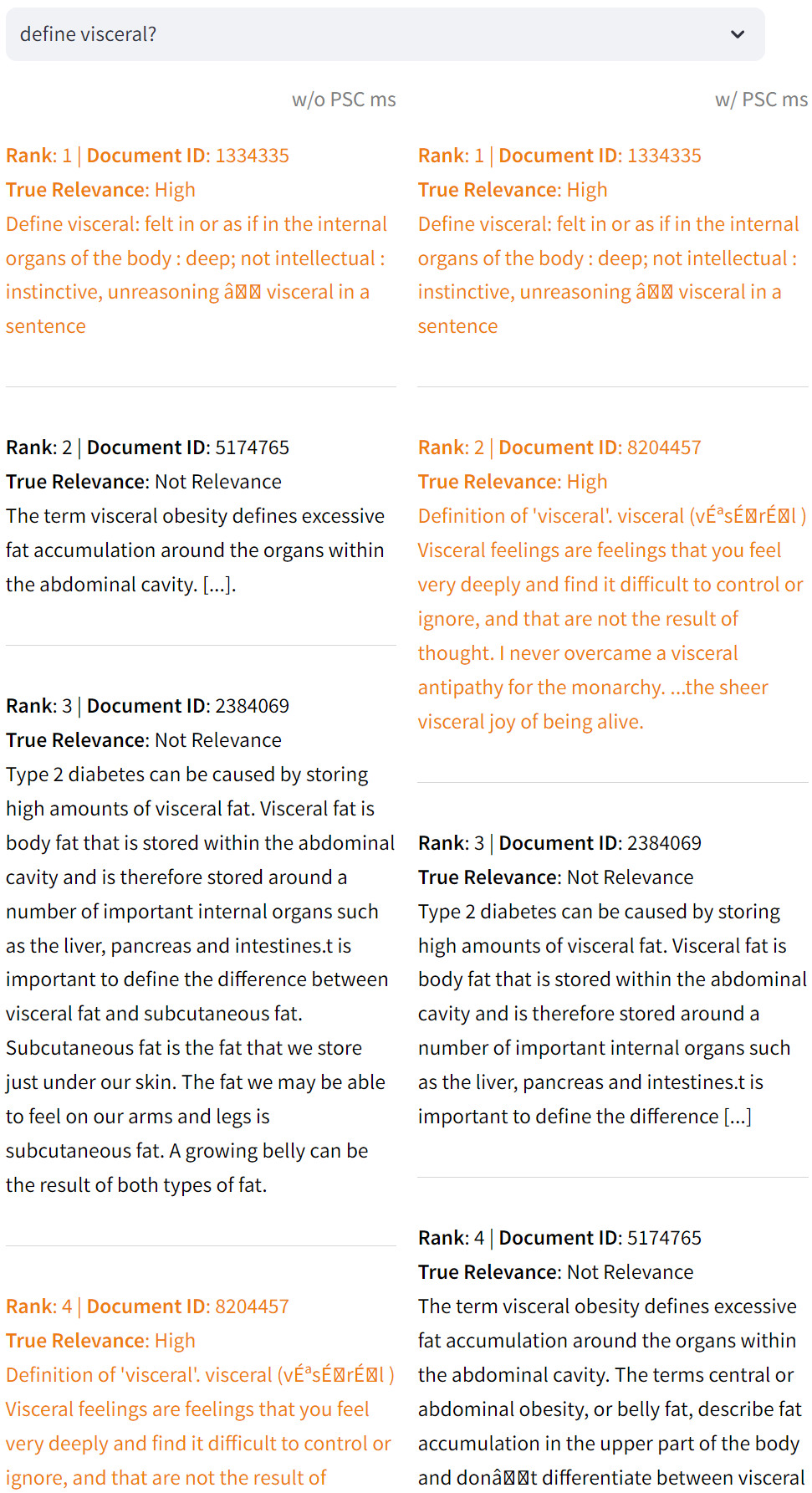}
        \caption{Results reranked by GPT-3.5. Both PSC and conventional inference rank the first document the same, but PSC correctly ranks document \#8204457 higher (second vs.\ fourth).}\label{fig:appendix:ex2}
    \end{subfigure}
    \hfill
    \begin{subfigure}[t]{1\columnwidth}
        \centering
        \includegraphics[width=\columnwidth]{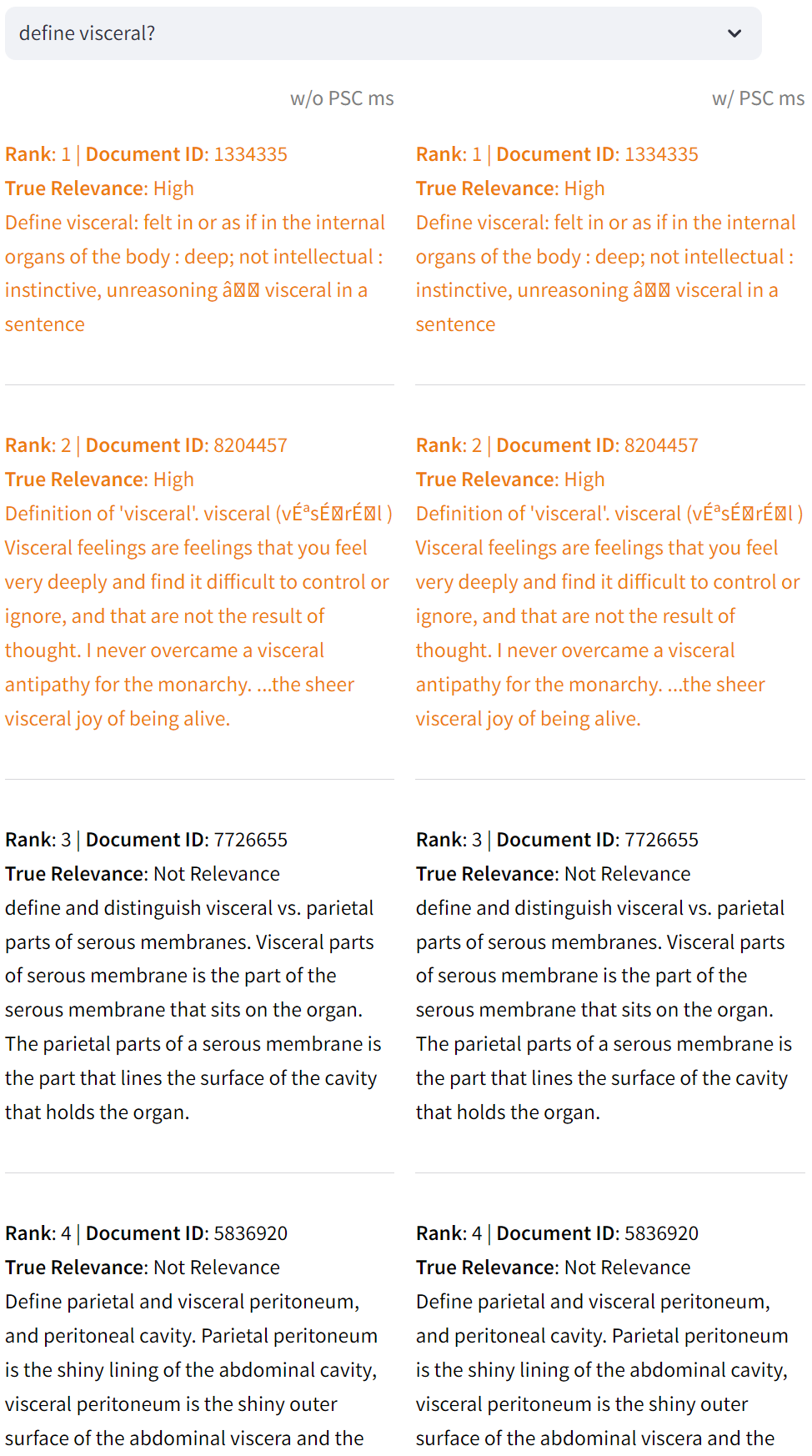}
        \caption{Results reranked by GPT-4. Compared to the previous example, PSC results in no difference.}\label{fig:appendix:ex3}
    \end{subfigure}
    \caption{The DL19 query ``define visceral?'' with relevant documents reranked without PSC on the left and with PSC on the right of each subfigure.}
\end{figure*}
\begin{figure}[t!]
    \includegraphics[width=\columnwidth]{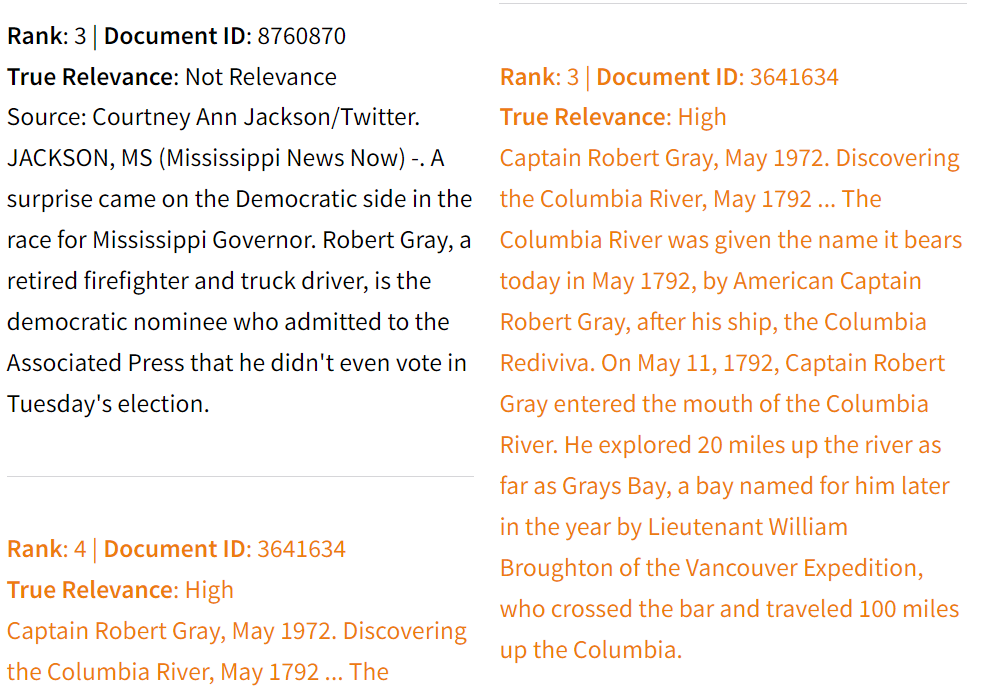}
    \caption{The DL19 query ``who is robert gray'' with relevant documents reranked without PSC on the left and with PSC on the right, with GPT-4 as the model.}\label{fig:appendix:ex1}
\end{figure}

\parheader{Dataset settings}
We made a few further considerations in designing WordSort and MathSort.
To add difficulty to word sorting, for each example we randomly mixed five consecutively ordered words in the English language with five randomly picked ones, e.g., mixing ``aardvark, aaron, abacus ...'' with ``dog, cat, shrew ...''
On MathSort, we ensured that all expressions evaluated to unique values within an example (of 10 expressions).
None of our lists for any task were duplicates.

\subsection{Passage Reranking Task}\label{sec:appendix:passage-reranking-task}

\autoref{tab:passage-rerank-detailed} lists the full prompts for our passage reranking task, following prior art precisely~\cite{sun2023chatgpt, pradeep2023rankvicuna}.
We used the same output extraction procedure from the official codebases as well, ensuring a faithful comparison.

For our OpenAI GPT endpoints, we deployed GPT-3.5 Turbo (version 0613) and GPT-4 on Azure.
In total, at the current public price of \$0.002 and \$0.03 per one thousand tokens,\footnote{\url{https://openai.com/pricing}} we estimate a cost of \$100--200 USD to reproduce the GPT passage ranking results in their entirety, with GPT-4 consuming most of it.

\section{Qualitative Examples}~\label{sec:appendix:qualitative-examples}
We present qualitative examples of our approach on DL19 in Figures~\ref{fig:appendix:ex2}, \ref{fig:appendix:ex3}, and \ref{fig:appendix:ex1}.
In Figures~\ref{fig:appendix:ex2} and \ref{fig:appendix:ex3}, we compare the outputs of GPT-3.5 and GPT-4 with PSC, fixing the query to ``define visceral?''
We find that PSC improves GPT-3.5 but not GPT-4, since GPT-4's original output is already correct, providing visual evidence for why PSC attains more gains on GPT-4 than on GPT-3.5.
In Figure~\ref{fig:appendix:ex1}, GPT-4 with PSC ranks the third document (\#3641634) correctly higher (right) than GPT-4 without PSC (left).
In summary, these illustrations suggest that the quantitative improvements of PSC are not merely illusory.
\vfill
\begin{table*}[t]
    \small
    \setlength{\tabcolsep}{6pt}
    \centering
    \begin{tabular}{lllcccc}
    \toprule[1pt]
    \multirow{2.3}{*}{First Stage} & \multirow{2.3}{*}{Top-$k$} &  \multirow{2.3}{*}{Method} & \multicolumn{2}{c}{\bf TREC-DL19} & \multicolumn{2}{c}{\bf TREC-DL20} \\[-0.5ex]
     \cmidrule(lr){4-5} \cmidrule(lr){6-7}
    & & & Original & Our PSC & Original & Our PSC \\
    \midrule[1pt]
    BM25 & 20  & (1) Single (GPT-3.5, Reversed) & 55.92 & \textbf{62.88} & 52.66 & \textbf{59.59}\\
             & 20 & (2) Single (GPT-4, Reversed) & 64.04 & \textbf{65.60} & 60.20 & \textbf{62.27}\\
         & 100 & (3) RankGPT (GPT-3.5, Reversed) & 56.76 & \textbf{57.32} & 51.03 & \textbf{55.73}\\
             & 100 & (4) RankGPT (GPT-4, Reversed) & 67.83 & \textbf{69.63} & 64.92 & \textbf{65.89}\\
    \bottomrule[1pt]
    \end{tabular}
    \caption{nDCG@10 results on DL19 and 20.}
    \label{tab:results-ir-addl}
\end{table*}

\section{Supplementary Results and Discussion}
During the peer review process of this paper, our reviewers helpfully suggested experiments to further bolster the rigor of our claims.
We explicitly include most of them here, with the remaining feedback incorporated into the related work section.
\begin{table}[t]\small
    \setlength{\tabcolsep}{1.5pt}
    \centering
    \begin{tabular}{lccc}
    \toprule[1pt]
    Method & \textsc{MathSort} & \textsc{WordSort} & \textsc{GSM8KSort} \\
    \midrule[1pt]
    GPT-3.5 (PRP) & 46.7 & 82.2 & 64.0 \\
    GPT-4 (PRP) & 73.3 & 83.9 & 79.9 \\
    \midrule
    GPT-3.5 (PSC) & 75.2 & 88.1 & 88.4 \\
    GPT-4 (PSC) & \textbf{89.6} & \textbf{92.0} & \textbf{90.5} \\
    \bottomrule[1pt]
    \end{tabular}
    \caption{Pairwise ranking prompting versus permutation self-consistency on the sorting tasks.}
    \label{tab:results-sort-addl}
\end{table}

\begin{table}[t]\small
    \setlength{\tabcolsep}{1.5pt}
    \centering
    \begin{tabular}{lccccc}
    \toprule[1pt]
    Method & \textsc{Math} & \textsc{Word} & \textsc{GSM8K} & \textsc{DL19} & \textsc{DL20} \\
    \midrule[1pt]
    GPT-3.5 (Orig.) & 64.0 & 85.9 & 82.1 & 68.00 & 62.08 \\
    GPT-3.5 (Borda) & 74.6 & 87.9 & 88.1 & 70.09 & 62.54 \\
    GPT-3.5 (Our PSC) & \textbf{75.2} & \textbf{88.1} & \textbf{88.4} & \textbf{70.77} & \textbf{62.70} \\
    \midrule
    GPT-4 (Orig.) & 83.5 & 89.9 & 88.4 & 75.00 & 70.36 \\
    GPT-4 (Borda) & 89.2 & 91.5 & 90.4 & 75.23 & 70.62 \\
    GPT-4 (Our PSC) & \textbf{89.6} & \textbf{92.0} & \textbf{90.5} & \textbf{75.66} & \textbf{71.00} \\
    \bottomrule[1pt]
    \end{tabular}
    \caption{Comparisons to using bootstrapping/Borda count as the aggregation algorithm.}
    \label{tab:results-borda-addl}
\end{table}

\subsection{Sorting Tasks}
In \autoref{tab:results-sort-addl}, we demonstrate that our PSC approach outperforms PRP-10 by 15 points in Kendall's tau on average, with higher gains on GPT-3.5.
We chose PRP-10 because it most closely matches ours in computation time.
Overall, these findings are in line with our conclusions on the passage reranking task, as shown in \autoref{tab:results-ir}.

\subsection{Passage Reranking Task}
We additionally conducted experiments on reversing input orders for passage reranking.
Note that it is inapplicable to sorting because the underlying dataset is unordered, so reversing the input would not affect the results.
Shown in \autoref{tab:results-ir-addl}, our PSC method improves the result by an average of 3.2 points;\ thus, it successfully mitigates position bias regardless of the order used in the sliding window.

Finally, we show that PSC outperforms the bootstrapping technique (or ``Borda count,'' as the rank aggregation literature calls it) from \citet{hou2024large}.
Presented in Table~\ref{tab:results-borda-addl}, our method consistently outperforms bootstrapping on GPT-3.5 and GPT-4, possibly because of Kemeny ranking's theoretical optimality.
The gains are roughly equal between GPT-3.5 (0.39 average point increase) and GPT-4 (0.42 points).
We conclude that the choice of rank aggregation algorithm matters.

\subsection{Computational Burden}
\label{sec:appendix:computational-burden}
Finally, we discuss the difference in time and computational cost of the proposed method compared to other baselines.
We concede computation time is a general limitation of many self-consistency-style works, as we acknowledge in our limitations section.
Our advantage is that PSC is embarrassingly parallel and scales horizontally with ease, e.g., for a choice of five repetitions, spinning up five instances will be roughly equivalent to the original baseline of one.
Furthermore, our Kemeny-optimal aggregation method is virtually instantaneous for practical sample sizes (less than 0.05 CPU seconds).
This contrasts with methods such as PRP that necessitate, say, 20--200 sequential calls for a list size of 10 rather than our fully parallelizable 5--20.

Thus, even though in theory our method is asymptotically linear, in practice the big-O constant can be made ``small'' by horizontal scaling, assuming the presence of parallel computing.
As a rough quantitative comparison, our experiments run 20 parallel calls to multiple deployments of GPT-3.5 and GPT-4, incurring a running time of no more than 25\% (in addition to) a single call.

\end{document}